\documentclass[final,12pt]{clear2024} %

\title{Fundamental Properties of Causal Entropy and Information Gain}
\usepackage{times}

\clearauthor{%
  \Name{Francisco N. F. Q. {Simoes}} \Email{f.simoes@uu.nl}\\
  \Name{Mehdi Dastani} \Email{m.m.dastani@uu.nl}\\
  \Name{Thijs {van Ommen}} \Email{t.vanommen@uu.nl}\\
  \addr Department of Information and Computing Sciences, Utrecht University }

\usepackage{amsmath, amssymb, mathtools} \usepackage{enumitem, cleveref,
  nameref} %
\usepackage{cancel} \usepackage{nicefrac}
\usepackage{stmaryrd} %

\usepackage{soul} %
\soulregister\cite7 %
\soulregister\Citet7
\soulregister\Citep7
\soulregister\citet7
\soulregister\citep7
\soulregister\ref7
\soulregister\cref7
\soulregister\Cref7

\usepackage{booktabs}
\usepackage[bottom]{footmisc} %
\usepackage{xargs} %

\let\todo\undefined %
\usepackage[colorinlistoftodos,prependcaption,textsize=tiny]{todonotes}

\DeclareMathOperator{\E}{\mathbb{E}} 
\DeclareMathOperator{\cC}{\mathcal{C}} 
 \newcommand{\dop}{\mathit{do}}
 \newcommand{\PA}{\mathrm{PA}}
 
\newcommand{\veryshortrightarrow}{\!\!\shortrightarrow\!\!}

\newcommand{\ccolon}{\vcentcolon=}

\newcommandx{\warning}[2][1=]{\todo[linecolor=red,backgroundcolor=red!25,bordercolor=red,#1]{#2}}
\newcommandx{\mynote}[2][1=]{\todo[linecolor=blue,backgroundcolor=blue!25,bordercolor=blue,#1]{#2}}
\newcommandx{\myquestion}[2][1=]{\todo[linecolor=purple,backgroundcolor=purple!25,bordercolor=purple,#1]{#2}}
\newcommandx{\original}[2][1=]{\todo[linecolor=green,backgroundcolor=green!25,bordercolor=green,#1]{OG!:
    #2}}

\begin{document}

\maketitle

\begin{abstract}
  Recent developments enable the quantification of \emph{causal control} given a
  structural causal model (SCM). This has been accomplished by introducing
  quantities which encode changes in the entropy of one variable when
  intervening on another. These measures, named causal entropy and causal
  information gain, aim to address limitations in existing
  information theoretical approaches for machine learning tasks where causality
  plays a crucial role. They have not yet been properly mathematically studied.
  Our research contributes to the formal understanding of the notions of causal entropy and causal information gain by establishing and analyzing fundamental properties of these concepts, including bounds and chain rules.
  Furthermore, we elucidate the relationship between causal entropy
  and stochastic interventions.
  We also propose definitions for causal conditional entropy and causal conditional information gain.
  Overall, this exploration paves the way for enhancing causal machine learning tasks through the study of recently-proposed information theoretic quantities grounded in considerations about causality.%
\end{abstract}


\section{Introduction}
\label{sec:introduction}
Information theoretical quantities are ubiquitous in machine learning.
Cross-entropy losses are commonly used in deep learning
\citep{goodfellow2016deep}.
Decision tree learning algorithms commonly use mutual information (often called information gain in that context) to decide what variable to split at each step \citep{gareth2013introduction} --- see for example the ID3 \citep{quinlan1986induction} and CR4.5 \citep{quinlan1993c4} algorithms.
In \citet{haarnoja2018soft}, reinforcement learning tasks see their stability and robustness improved when the agents learn by maximizing not only their expected rewards but also the entropy of their policies.
Also in \citet{seitzer2021reinforcement}, mutual information is used for reinforcement learning tasks to measure whether or not an agent's action has ``causal influence'' on a state, and they use this quantity to decide if said action has ``control'' over that state.
In \citet{achille2018information}, an optimal representation $\mathbf{Z}$ of a random vector $\mathbf{X}$ for a given task $\mathbf{Y}$ is learned by minimizing the information bottleneck lagrangian, which is a difference between the mutual informations $I(\mathbf{X}; \mathbf{Z})$ and $I(\mathbf{Y}; \mathbf{Z})$ \citep{tishby2000information}, the idea being that $\mathbf{Z}$ should keep as little information about $\mathbf{X}$ as possible while retaining as much information about $\mathbf{Y}$ as possible.
In \citet{holtgen2021encoding}, the aforementioned information bottleneck lagrangian is used in the loss function of an autoencoder in order to learn a causally-relevant representation for a given task.
The goal of this approach is to enable us to intervene on the representation instead of on the original variables without losing much control over the task variable.

In the cases just described the existence of confounders or selection bias could lead to misleading results if one ascribes a causal interpretation.
This is a common issue when using standard information
theoretical quantities in situations that require consideration of the
underlying causal relationships. A version of mutual information which takes
into account the causal structure of the system would solve this problem.
Preliminary work in this direction was recently done by \citet{simoes2023causal}
in the context of interpretable machine learning and inspired by the earlier
philosophical work of \citet{griffiths2015specificity}.
The former define
``causal entropy'' and ``causal information gain'' and study their relationship
with total causal effect. They also argue that causal information gain provides
an adequate measure of causal control, so that it can be used when
deciding which variables in an SCM provide more control over a chosen target
variable. However, a proper mathematical study of the properties of causal
entropy and causal information gain is missing.

In this paper, we relate causal entropy with other quantities such as conditional entropy and post-stochastic intervention entropy, revealing potentially more convenient approaches to its computation as
well as new interpretations of this quantity.
We also define conditional causal entropy and conditional causal information gain. Additionally, we derive
fundamental properties of both causal entropy and causal information gain,
drawing upon analogous results from information theory.

The novelty of our work consists of deriving key properties
of causal entropy and causal information gain for the first time, and defining conditional causal entropy and information gain.
Concretely:
\begin{itemize}
  \item We show that, perhaps unexpectedly, causal entropy is not the same as the entropy after a stochastic intervention, and establish formal relations between causal entropy and stochastic interventions.
        Furthermore, we show that causal entropy can be seen as a post-stochastic-intervention conditional entropy.
  \item We check that causal entropy is non-negative, but that, surprisingly, causal information
        gain can be negative.
        We also find upper bounds on the causal entropy, including an
        independence bound mirroring the independence bound on standard entropy.
        We check that, unexpectedly, the causal version of the data processing inequality does not hold.
  \item We define conditional causal entropy and conditional causal information gain for the first time.
        We use these to derive chain rules for both causal entropy and causal information gain.
        Finally, we discuss how alternative causal versions of conditional information gain are possible, and study one in particular.
\end{itemize}

This paper is organized as follows. \Cref{sec:formal-setting} introduces the
assumptions and the definitions of information theoretical quantities that will be
used throughout the paper, including causal entropy and
causal information gain. In \Cref{sec:causal-entr} we discuss how causal
entropy differs from the entropy of the post-stochastic-intervention
distribution.
We study how causal entropy can be linked to post-stochastic interventions, resulting in measures of causal entropy which can be both elucidating and easier to compute in practice.
We also establish lower and upper bounds on the causal entropy, define conditional causal entropy, and present some crucial findings that culminate in a chain rule for causal entropy.
\Cref{sec:causal-inf-gain} lays down fundamental properties of causal information gain, including a chain rule.
It also introduces conditional causal information gain and discusses its interpretation, along with the consideration of an alternative causal extension of conditional mutual information termed post-intervention mutual information.
\Cref{sec:related-work} compares the definitions
and results in this text with that of other work that has been done before. In
\Cref{sec:disc-concl} we discuss the results obtained in this work and propose
future avenues of research.
The proofs of all the results presented in this paper can be found in the appendix.

\section{Formal Setting}
\label{sec:formal-setting}

Many basic concepts from causal inference and information theory are used throughout this paper.
For completeness, we include the necessary definitions from causal inference in \Cref{sec:struct-caus-models}.
All random variables are henceforth assumed to be discrete and have finite
range.
In this paper, a ``random variable'' can be a random vector.
In cases where we want to restrict ourselves to random variables which are in fact random vectors, boldface is used.
\emph{E.g.} $X$ can be both a single-value random variable or a random vector, while $\mathbf{X}$ must be a random vector.
Furthermore, the symbols of the form $X$, $Y$, $Z$, $X_{i}$, $Y_{i}$ or
$Z_{i}$ are taken to be endogenous variables of some SCM~$\mathcal{C}$.

In this section we present the definitions from information theory which are
necessary for the rest of this paper. We also include the definitions of causal
entropy and causal information gain.

\subsection{Entropy and Mutual Information}
\label{sec:entr-mutu-inform}
In this subsection we will start by stating the definitions of entropy, conditional
entropy and mutual entropy. In the interest of space, we will not try to
motivate these definitions.
For more information, see \citet{thomas2006information}.
\begin{definition}[Entropy and Cond. Entropy
  \citep{thomas2006information}]
  \label[definition]{def:entr}
  Let $X$ be a discrete random variable with range $R_{X}$ and $p$ be a
  probability distribution for $X$. The \emph{entropy of $X$ w.r.t. the
    distribution $p$} is\footnotemark \footnotetext{In this article, $\log$
    denotes the logarithm to the base $2$.}
  \begin{equation}
    \label{eq:entr}
    H_{X \sim p}(X) \vcentcolon= -\sum_{x\in R_{X}} p(x) \log p(x).
  \end{equation}
  Entropy is measured in $\mathrm{bit}$.
  If the context suggests a canonical probability distribution for $X$, one can write $H(X)$ and refers to it simply as the \emph{entropy of $X$}. \\
  The \emph{conditional entropy} $H(Y\mid X)$ of $Y$ conditioned on $X$ is the
  expected value w.r.t. $p_{X}$ of the entropy
  $H(Y \mid X=x)\vcentcolon=H_{Y\sim p_{Y\mid
      X=x}}(Y)$:%
  \begin{equation}
    \label{eq:cond-entr}
    H(Y\mid X) \vcentcolon= \E_{x\sim p_{X}} \left[ H(Y \mid X=x) \right].
  \end{equation}
\end{definition}
This means that the conditional entropy $H(Y \mid X)$ is the entropy of $H(Y)$
that remains on average if one conditions on $X$.
\begin{remark}
  Notice that $H(Y \mid X=x)$ is seen as a function of $x$ and the expected
  value in \Cref{eq:cond-entr} is taken over the random variable $x$ with
  distribution $p_{X}$. This disrespects the convention that random variables
  are represented by capital letters, but preserves the convention that the
  specific value conditioned upon is represented by a lower case letter.
  We will follow the common practice and opt to use lower case letters for random variables in these cases.
\end{remark}

There are two common equivalent ways to define mutual information (often called
information gain).

\begin{definition}[Mutual Information and Cond. Mutual Information
  \citep{thomas2006information}]
  \label[definition]{def:mutual-information}
  Let $X$ and $Y$ be discrete random variables with ranges $R_{X}$ and $R_{Y}$
  and distributions $p_{X}$ and $p_{Y}$, respectively. The \emph{mutual
    information} between $X$ and $Y$ is
  \begin{equation}
    \label{eq:mi-independence-form}
    I(X; Y) := \!\!\!\! \sum_{x, y \in R_{X} \times R_{Y}} \!\!\!\! p_{X, Y}(x, y) \log \frac{p_{X, Y}(x, y)}{p_{X}(x) p_{Y}(y)}.
  \end{equation}
  Or equivalently:
  \begin{align}
    \begin{split}
      \label{eq:mi-entr-form}
      I(X; Y) &:= H(Y) - H(Y \mid X) \\
      &= H(X) - H(X \mid Y).
    \end{split}
  \end{align}
  Let $Z$ be another discrete random variable. The \emph{conditional mutual
    information} between $X$ and $Y$ conditioned on $Z$ is:
  \begin{align}
    \begin{split}
      \label{eq:cond-mut-info}
      I(X; Y \mid Z) &:= H(Y \mid Z) - H(Y \mid X, Z) \\
                    &= H(X \mid Z) - H(X \mid Y, Z).
    \end{split}
  \end{align}
\end{definition}

The view of mutual information as entropy reduction from \Cref{eq:mi-entr-form}
is the starting point for the definition of causal information gain.

\subsection{Causal Entropy and Causal Information Gain}
We will now define causal entropy and causal information gain.
See \citet{simoes2023causal} for a thorough discussion about these concepts.
The causal entropy of $Y$ for $X$ is the entropy of $Y$ that is left, on
average, after one atomically intervenes on $X$. It is defined in a manner
analogous to conditional entropy (see \Cref{def:entr}). Concretely, causal
entropy is the average uncertainty one has about $Y$ if one sets $X$ to $x$ with
probability $p_{X'}(x)$, where $X'$ is a new auxiliary variable with the same
range as $X$ but independent of all other variables, including $X$.
\begin{definition}[Causal Entropy, $H_{c}$ \citep{simoes2023causal}]
  Let $Y$, $X$ and $X'$ be random variables such that $X$ and $X'$ have the same
  range and $X'$ is independent of all variables in $\cC$. We say that $X'$ is
  an \emph{intervention protocol} for $X$.
  The \emph{causal entropy} $H_{c}(Y\mid \dop(X \sim X'))$ of $Y$ for $X$ given the
  intervention protocol $X'$ is the expected value w.r.t. $p_{X'}$ of
  the entropy
  $H(Y \mid \dop(X = x)) \vcentcolon= H_{Y \sim p_{Y}^{\dop(X=x)}}(Y)$ of the
  interventional distribution $p_{Y}^{\dop(X=x)}$. That is:
  \begin{equation}
    \label{eq:caus-cond-entr}
    H_{c}(Y\mid \dop(X \sim  X')) \vcentcolon= \E_{x\sim p_{X'}} \left[ H(Y \mid \dop(X=x)) \right].
  \end{equation}
\end{definition}

Causal information gain extends mutual information/information gain to the
causal context. While mutual information between two variables $X$ and $Y$ is
the average reduction in uncertainty about $Y$ if one observes the value of $X$
(see \Cref{eq:mi-entr-form}), the causal information gain of $Y$ for $X$ is
the average decrease in the entropy of $Y$ after one atomically intervenes on
$X$ (folowing an intervention protocol $X'$).
\begin{definition}[Causal Information Gain, $I_{c}$ \citep{simoes2023causal}]
  \label[definition]{def:Ic}
  Let $Y$, $X$ and $X'$ be random variables such that $X'$ is an intervention
  protocol for $X$. The \emph{causal information gain}
  $I_{c}(Y\mid \dop(X \sim X'))$ of $Y$ for $X$ given the intervention protocol
  $X'$ is the difference between the entropy of $Y$ w.r.t. its prior and the
  causal entropy of $Y$ for $X$ given the intervention protocol $X'$. That is:
  \begin{equation}
    \label{eq:caus-cond-entr}
    I_{c}(Y \mid \dop(X \sim X')) \vcentcolon= H(Y) - H_{c}(Y\mid \dop(X \sim  X')).
  \end{equation}
\end{definition}
The causal information gain of $Y$ for $X$ was proposed in \citet{simoes2023causal} as a measure of the ``(causal) control that variable $X$ has over the variable $Y$''.
This is a qualitative concept used in the philosophy of science literature \citep{pocheville2015comparing} and
defined in \citet{simoes2023causal} as the reduction of uncertainty about $Y$ that results from intervening on $X$.
This is precisely what causal information gain measures, by construction.
It is important to note that a common measure of causal strength such as average causal effect (ACE) would not be a suitable measure of causal control. Indeed, the uncertainty about $Y$ can be reduced by intervening on $X$ while maintaining the average of $p^{do(X=1)}_{Y}$ the same as that of  $p^{do(X=0)}_{Y}$, yielding an ACE of zero even though uncertainty is reduced by intervening on $X$.

\section{Properties of Causal Entropy}
\label{sec:causal-entr}
We start this section by showing that causal entropy is distinct from the
entropy after a stochastic intervention $\dop(X=X')$.
We study the relation of causal entropy with other quantities, providing new insights about causal entropy and establishing some first basic properties.
This section ends with the definition of conditional causal entropy and the derivation of a chain rule for causal entropy.

\subsection{Comparison with Entropy After a Stochastic Intervention}
\label{sec:comparison-with-post}
One could think that $H_{c}(Y\mid \dop(X\sim X')) = H(Y\mid \dop(X=X'))$:
both are entropies of $Y$ resulting from making $X$ follow the distribution $p_{X'}$, albeit through two distinct procedures.
While these may appear identical, this is, in reality, not accurate.
In other words, the average uncertainty about $Y$ after
\emph{atomically} intervening on $X$ by setting $X=x$ with probability $p_{X'}$
is not the same as the average uncertainty after \emph{stochastically}
intervening on $X$ by setting $X$ to $X'$. \Cref{ex:h_differs_from_hc} will
illustrate this. The underlying reason for this difference will be made clear by
\Cref{prop:Hc-as-E-cov-eff}.

\begin{example} %
  \label{ex:h_differs_from_hc}
  We will look at an example where
  $H_{c}(Y\mid \dop(X\sim X')) \ne H(Y\mid \dop(X=X'))$.
  Consider the SCM over the variables $X, Y$ with ranges $R_{X} = \{0,1\}$ and
  $R_{Y} = \{0,1,2\}$, characterized by the following structural
  assignments and noise distributions:
  \begin{equation}
    \label{eq:scm-example-h-diff-from-hc}
    \begin{cases}
      f_{X}(N_{X}) = N_{X} \\
      f_{Y}(X, N_{Y}) = X + N_{Y} \\
      N_{X}, N_{Y} \sim \mathrm{Bern}(\frac{1}{2})
    \end{cases}
  \end{equation}
  Notice that the causal graph is then simply $X \rightarrow Y$.
  Further, let $X'$ be an intervention protocol for $X$ with
  $p_{X'} = \mathrm{Bern}(\frac{1}{3})$. The atomic and stochastic interventions
  on $Y$ can then be written as in \Cref{tab:post-int-dist-example}.
  \begin{table}[h!]
    \centering
    \begin{tabular}{llll}
      \toprule %
      $Y$ & $p_{Y}^{\dop(X=0)}$ & $p_{Y}^{\dop(X=1)}$ & $p_{Y}^{\dop(X=X')}$\\
      \midrule %
      $0$ & $\nicefrac{1}{2}$ & $0$ & $\nicefrac{1}{3}$\\
      $1$ & $\nicefrac{1}{2}$ & $\nicefrac{1}{2}$ & $\nicefrac{1}{2}$\\
      $2$ & $0$ & $\nicefrac{1}{2}$ & $\nicefrac{1}{6}$\\
      \bottomrule %
    \end{tabular}
    \caption{Post-intervention distributions for $Y$ in
      \Cref{ex:h_differs_from_hc}. (Computed in \Cref{comp:h-from-table}).}    \label{tab:post-int-dist-example}
  \end{table}
  Hence\footnotemark $H_{c}(Y \mid \dop(X \sim X')) = 1 (\mathrm{bit})$ and:
  \begin{equation}
    H(Y\mid \dop(X=X'))
    =\frac{1}{3}\underbrace{\log 3}_{>1} +\, \frac{1}{2} + \frac{1}{6}\underbrace{\log 6}_{>1} > 1 (\mathrm{bit}).
  \end{equation}
  \footnotetext{The details of the computations of these entropies can be found
    in \Cref{comp:h-from-table}.} Thus in particular
  $H(Y\mid \dop(X=X')) \ne H_{c}(Y\mid \dop(X\sim X'))$ in this example.
\end{example}

\begin{remark}[Intuition for \Cref{ex:h_differs_from_hc}]
  \label{rem:example-explanation-hc-differs-from-hc}
  The difference between causal and post-intervention entropies observed in
  \Cref{ex:h_differs_from_hc} stems from the fact that, since $Y = X + N_{Y}$,
  fixing X renders the distribution of Y equal in shape to that of $N_{Y}$.
  Hence the post-atomic intervention entropies $H_{Y\sim p_{Y}^{\dop(X=x)}}(Y)$ are all the same
  (1 bit), so that averaging just gives 1 bit. Notice that this does not depend
  on the choice of $p_{X'}$. Now, $H(Y\mid \dop(X=X'))$ is the entropy of the
  distribution of the sum of random variables $X'+N_{Y}$, meaning that in this
  case the shape of the distribution of $Y$ is changed, not simply shifted.
\end{remark}

\subsection{Alternative Views of Causal Entropy}
\label{sec:prop-caus-entr}
We will now see that causal entropy can be written as an expected value over a post-stochastic intervention joint distribution.
It is often useful to have an expression for a statistic as an average w.r.t. the joint distribution of the variables involved.
This enables, for instance, the straightforward construction of an estimator for the statistic, known as the ``plug-in estimator'', achieved by substituting the joint distribution figuring in the expected value by the empirical joint distribution \mbox{\citep{wasserman2004all}}.
\begin{proposition}[Causal Entropy as a Single Expected Value]
  \label[proposition]{prop:cH-single-E}
  The causal entropy of $Y$ given the intervention protocol $X'$ for $X$ can be
  written as an expected value w.r.t. the post-stochastic-intervention
  distribution $p_{X,Y}^{\dop(X=X')}$ as follows:
  \begin{align}
    \begin{split}
      \label{eq:hc_single_E_expression_proposition}
      H_{c}(Y&\mid \dop(X\sim X')) =- \E_{x,y \sim p_{X,Y}^{\dop(X=X')}} \Big[ \log p(y \mid \dop(X=x)) \Big].
    \end{split}
  \end{align}
\end{proposition}

We will now see how causal entropy relates with post-stochastic intervention entropies.
We use $H(Y \mid X=x, \dop(X=X'))$ as notation for the entropy $H_{Y \sim p_{Y \mid X=x}^{\dop(X=X')}}(Y)$ of the covariate-specific effect\footnotemark $p_{Y \mid X=x}^{\dop(X=X')}$.
\footnotetext{The term ``covariate-specific effect'' is commonly used when conditioning on a variable distinct from the intervened variable. In this paper we use the term also when the conditioned variable coincides with the intervened variable.}
That is, $H(Y \mid X=x, \dop(X=X'))$ is the entropy resulting from conditioning on $X=x$ \emph{after} having performed the stochastic intervention $\dop(X=X')$.
\begin{proposition}[Causal Entropy as Average Entropy of Covariate-Specific Effects]
  \label[proposition]{prop:Hc-as-E-cov-eff}
  The causal entropy of $Y$ given the intervention protocol $X'$ for $X$ can be seen as the expected value w.r.t. $x\sim p_{X'}$ of the entropies
  $H(Y \mid X=x, \dop(X=X'))$ of the
  covariate-specific effects $p_{Y \mid X=x}^{\dop(X=X')}$.
  That is:
  \begin{align}
    \begin{split}
    \label{eq:Hc_is_average_covspeceffect}
      H_{c}(Y\mid \dop(X \sim  X')) &= \E_{x \sim p_{X'}}\left[H(Y\mid X = x, \dop(X=X'))\right] \\
                                   &= H(Y\mid X, \dop(X=X'))
    \end{split}
  \end{align}
  where the notation used in the second equality is analogous to the notation for \mbox{conditional entropy}.
\end{proposition}
This result further elucidates the origin of the difference discussed in \Cref{ex:h_differs_from_hc}.
Concretely, \Cref{prop:Hc-as-E-cov-eff} shows us that $H_{c}(Y\mid \dop(X \sim X'))$ is the
average of the entropies of $Y$ obtained from stochastically intervening (according to
$p_{X'}$) AND knowing the value that $X$ was set to, while
$H(Y\mid \dop(X = X'))$ is simply the entropy of $Y$ after performing the stochastic intervention -- which we can interpret as having performed an atomic intervention on $X$, but not knowing exactly which value $X$ was set to.

Notice that \Cref{prop:Hc-as-E-cov-eff} implies that the causal entropy is a \emph{bona fide} conditional entropy, where the conditioning is performed after a stochastic intervention setting $X=X'$.
Indeed, this is precisely the meaning of $H(Y \mid X, \dop(X=X'))$.
Consequently, we can harness known properties of conditional entropy to prove properties of causal entropy.

\subsection{Bounds on Causal Entropy}
\label{sec:comparison-hc-and-ind-bound}
Just like conditional entropy, causal entropy is a non-negative quantity. This
follows simply from the fact that causal entropy is an average of entropies,
which are themselves non-negative.
\begin{proposition}
  \label[proposition]{prop:hc-nonnegative}
  Causal entropy is non-negative.
\end{proposition}

The following result is not so much a property of causal entropy as it is the absence
of one -- namely, causal entropy is not necessarily smaller than the initial entropy.
This can be surprising, since it is in stark contrast with conditional entropy,
which is always less that the initial entropy.
Meaning that, on average, information about the conditioning variable $X$ can
never increase the uncertainty about $Y$ \citep{thomas2006information}.
In contrast, \Cref{prop:Hc-lessthan-H} tells us that, on average, intervening on $X$ can in
fact increase the uncertainty about $Y$.
\begin{proposition}
  \label[proposition]{prop:Hc-lessthan-H}
  For some SCMs and intervention protocols $X'$, the causal entropy of $Y$ for $X$ given an intervention protocol $X'$ is greater than the initial entropy of $Y$, \emph{i.e.} $H_{c}(Y \mid \dop(X \sim X')) > H(Y)$.
\end{proposition}
There is however an upper bound on causal entropy.
It follows immediately from \Cref{prop:Hc-as-E-cov-eff} that causal entropy, being related to the entropy after the stochastic intervention $\dop(X=X')$ by conditioning, cannot be greater than the latter.
\begin{corollary}
  \label[corollary]{cor:Hc-lessthan-stochH}
  The causal entropy of $Y$ for $X$ given an intervention protocol $X'$
  cannot be greater than the post-stochastic intervention entropy, \emph{i.e.}
  $ H_{c}(Y \mid \dop(X \sim X')) \le H(Y \mid \dop(X = X')). $
\end{corollary}

One can make use of the connection between causal
entropy and post-stochastic intervention conditional entropy together with the independence bound on entropy
\citep[Theorem 2.6.6]{thomas2006information} to derive an independence bound on
causal entropy.
\begin{proposition}[Independence Bound on Causal Entropy]
  \label[proposition]{prop:ind-bound}
  Let $\mathbf{Y}$ be a random vector of length $n_{Y}$. Then:
  \begin{equation}
    H_{c}(\mathbf{Y} \mid \dop(X \sim X')) \le
    \sum_{i=1}^{n_{Y}} H_{c}(Y_{i} \mid \dop(X \sim X')).
  \end{equation}
  and equality holds if and only if the $Y_{i}$ are independent.
\end{proposition}

\subsection{Conditional Causal Entropy and the Chain Rule}
\label{sec:cond-caus-entr}
We can of course mix intervening and conditioning. In this section we will start
by defining conditional causal entropy $H_{c}(Y \mid Z, \dop(X\sim X'))$, which
will capture the uncertainty that we have about $Y$ given that we intervened on
$X$ according to the intervention protocol $X'$ and then conditioned on $Z$.
We will then see that, unsurprisingly, conditioning reduces causal entropy on average,
and that, just like causal entropy, conditional causal entropy can also be seen
as a conditional entropy after a stochastic intervention. We will conclude this
section with a chain rule for causal entropy.
\begin{definition}[Conditional Causal Entropy] %
  \label[definition]{def:cond-hc}
  Let $X$, $Y$ and $Z$ be endogenous variables of an SCM $\mathcal{C}$ and
  $x \in R_{X}$. The \emph{atomic conditional causal entropy}
  $H_{c}(Y \mid Z, \dop(X=x))$ of $Y$ conditioned on $Z$ for the atomic
  intervention $\dop(X=x)$ is defined as the post-atomic intervention conditional entropy
  of $Y$ conditioned on $Z$, \emph{i.e.}:
  \begin{align}
    \begin{split}
      H_{c}(Y \mid Z, \dop(X=x)) &\vcentcolon= \E_{z \sim p_{Z}^{\dop(X=x)}} [H(Y \mid Z=z, \dop(X=x))]\\
                &= \E_{z \sim p_{Z}^{\dop(X=x)}} [H_{Y\sim p_{Y \mid Z=z}^{\dop(X=x)}} (Y)] .
    \end{split}
  \end{align}
  Moreover, the \emph{conditional causal entropy} $H_{c}(Y \mid Z, \dop(X \sim X'))$ of
  $Y$ conditioned on $Z$ given the intervention protocol $X'$ for $X$ is defined as the expected value given the intervention protocol $X'$ of the atomic conditional causal entropies, \emph{i.e.}:
  \begin{equation}
    H_{c}(Y \mid Z, \dop(X\sim X')) := \E_{x\sim p_{X'}}[H_{c}(Y \mid Z, \dop(X=x))].
  \end{equation}
\end{definition}
Notice that our definition of conditional causal entropy assumes that the intervention precedes the conditioning operation.
This assumption will hold throughout the paper.
A definition of a ``condition-first conditional causal entropy'' where one intervenes after conditioning would be more involved, and demands the use of counterfactuals.
This quantity should coincide with the conditional causal entropy here defined whenever $X$ does not have a causal effect on $Z$ (see \emph{e.g.} \Cref{fig:causal-graphs-examples}{(b)})--- in those cases, conditioning before or after intervening will result in the same distribution.
In the interest of space, we leave further discussion about this topic for future work.

Similarly to causal entropy, conditional causal entropy can also be regarded as the entropy of the distribution resulting from conditioning on $X$ following the stochastic intervention $\dop(X=X')$.
They differ only in the conditioning set.
\begin{proposition}[Conditional Causal Entropy as Conditional Entropy]
  \label[proposition]{prop:cond-hc-as-cond-entr}
  Let $X$, $Y$, $Z$ and $X'$ be as in \Cref{def:cond-hc}.
  Then:
  \begin{equation}
    H_{c}(Y \mid Z, \dop(X\sim X')) = H(Y \mid Z, X, \dop(X=X')).
  \end{equation}
\end{proposition}
Since this conditioning set is a superset of the conditioning set in $H(Y\mid X, \dop(X=X'))$, we can use the fact that conditioning cannot increase entropy to conclude that the conditional causal entropy cannot be larger than the causal entropy.
\begin{proposition}[Conditioning Reduces Causal Entropy]
  \label[proposition]{prop:cond-reduces-hc}
  The conditional causal entropy is never larger than the causal entropy.
  \emph{I.e.} for $X$, $Y$, $Z$ and $X'$ as in \Cref{def:cond-hc}, we have:
  \begin{equation}
    H_{c}(Y \mid Z, \dop(X \sim X')) \le H_{c}(Y \mid \dop(X \sim X')).
  \end{equation}
\end{proposition}
Utilizing \Cref{prop:Hc-as-E-cov-eff} and \Cref{prop:cond-hc-as-cond-entr} to express causal entropy and conditional causal entropy as conditional entropies enables us to use the standard chain rule for conditional entropy to derive a chain rule for causal entropy.
\begin{proposition}[Two-variable Chain Rule for Causal Entropy] %
  \label[proposition]{prop:hc-2-var-chain-rule}
  Let $X$, $Y$, $Z$ and $X'$ be as in \Cref{def:cond-hc}.
  Then:
  \begin{equation}
    \label{eq:Hc-chain-rule}
    H_{c}(Y, Z \mid \dop(X \sim X')) = H_{c}(Y \mid \dop(X \sim X')) + H_{c}(Z \mid Y, \dop(X\sim X')).
  \end{equation}
\end{proposition}
Similarly, we can obtain a chain rule specifically for random vectors by leveraging the general chain rule for the conditional entropy.
\begin{proposition}[Chain Rule for Causal Entropy]
  \label[proposition]{prop:chain-rule-hc}
  Let $\mathbf{Y}$ be a random vector of length $n_{Y}$ and $\mathbf{Y}_{<i} = Y_1, \ldots, Y_{i-1}$.
  Then:
  \begin{equation}
    H_{c}(\mathbf{Y} \mid \dop(X \sim X')) =
    \sum_{i=1}^{n_{Y}} H_{c}(Y_{i} \mid \mathbf{Y}_{<i}, \dop(X \sim X')).
  \end{equation}
\end{proposition}

\section{Properties of Causal Information Gain}
\label{sec:causal-inf-gain}
We start this section by noting basic properties of causal information gain.
We continue by defining conditional causal information gain and deriving a chain rule for causal information gain.
This section ends with a discussion about the interpretation of conditional causal information gain and post-intervention mutual information.
The latter is a distinct quantity from conditional causal information gain which serves as an alternative reasonable extension of conditional information gain in the context of causality.

\subsection{Immediate Properties of Causal Information Gain}
\label{sec:prop-caus-mutu}
A few properties of causal information gain can be immediately gleaned from
its definition.
In contrast with mutual information, causal information gain is
\emph{not} symmetric.
Similarly to causal entropy, one needs to specify an intervention protocol $X'$ which specifies the probability of each atomic intervention on $X$.
As shown in \citet{simoes2023causal}, if $X$ has no total effect on $Y$, then $I_{c}(Y \mid \dop(X \sim X')) = 0$ for any protocol $X'$.
In contrast with mutual information, causal information gain can be negative.
This is an immediate corollary of \Cref{prop:Hc-lessthan-H}.
\begin{corollary}
  \label[corollary]{cor:Ic-can-be-negative}
  For some SCMs and intervention protocols $X'$, the causal information gain of $Y$ given the intervention protocol $X'$ for $X$ is negative, \emph{i.e.} $I_{c}(Y \mid \dop(X \sim X')) < 0$.
\end{corollary}
Since KL divergences are non-negative, this means in
particular that causal information gain cannot be written as a KL divergence of
two distributions, again contrary to its non-causal counterpart.

One of the most important results in information theory is the data processing
inequality. It tells us that, for a Markov chain
$X \rightarrow Y \rightarrow Z$, $Z$ can never have more information about $X$
than $Y$ has \citep{thomas2006information}.
It is natural to wonder whether a similar result holds for causal information gain.
Specifically, one might ask if, for a causal chain $X \rightarrow Y \rightarrow Z$, the causal information gain of $Z$ for $X$ can never be larger than the causal information gain of $Y$ for $X$.
Such a proposition aligns with the intuitive notion that we have less control over variables ``farther away'' from the intervened variable.
Surprisingly, this is false.
No such causal data processing inequality holds for causal information gain.
To see this, we just need to devise an SCM whose causal graph is a chain $X \rightarrow Y \rightarrow Z$ and $I_{c}(Y \mid \dop(X \sim X')) < I_{c}(Z \mid \dop(X \sim X'))$.
This situation can arise, for instance, if performing an atomic intervention on $X$ still results in some uncertainty regarding a subset of values of $Y$, but every such value leads to the same $Z$.
\begin{example} %
  Consider a causal chain $X \rightarrow Y \rightarrow Z$ with ranges $R_{X} = \{x_{1}, x_{2}\}$, $R_{Y} = \{y_{1}, y_{2}, y_{3}\}$ and $R_{Z} = \{z_{1}, z_{2}\}$, and whose structural assignments and noise distributions are given by:
  \begin{equation*}
    \begin{array}{l}
      N_{X} \sim \mathcal{U}[R_{X}] \\
      N_{Y} \sim \mathcal{U}\{y_{1}, y_{2}\} \\
      X \ccolon N_{X} \\
    \end{array}
    \quad
      Y \ccolon
    \begin{cases}
      N_{Y}, \quad\!\!\! X = x_{1} \\
      y_{3}, \quad X = x_{2}
    \end{cases} \\
    \quad
    Z \ccolon
    \begin{cases}
      z_{1}, \quad Y = y_{1} \text{ or } Y = y_{2} \\
      z_{2}, \quad Y = y_{3}
    \end{cases} \\
  \end{equation*}
  Furthermore, we choose an intervention protocol $X'$ with a point mass at $x_{1}$, \emph{i.e.} $X' \sim \delta(x_{1})$ .
  Then $H(X) = H(Z) = 1$ and $H(Y) = 2\times \frac{1}{4} \log(4) + \frac{1}{2} \log(2) = \frac{3}{2}$.
  The relevant causal entropies are $H_{c}(Y \mid \dop(X\sim X')) = H_{c}(Y \mid \dop(X = x_{1})) = 1$ and $H_{c}(Z \mid \dop(X\sim X')) = H_{c}(Z \mid \dop(X = x_{1})) = 0$.
  Hence $I_{c}(Y \mid \dop(X\sim X')) = \frac{1}{2} < I_{c}(Z \mid \dop(X\sim X')) = 1$.
\end{example}

\subsection{Conditional Causal Information Gain and the Chain Rule}
\label{sec:cond-info-gain-and-chain-rule}
Recall the definition of conditional mutual information in \Cref{eq:cond-mut-info}.
It captures how much the uncertainty about $Y$ is reduced on average after
one \emph{observes} $X$, \emph{if one knows $Z$}.
Similarly, the
``conditional causal information gain'' will be defined such that it captures
how much the uncertainty about $Y$ is reduced on average if one \emph{sets} $X$ to $x$
with probability $p_{X'}(x)$, \emph{and\footnotemark one knows $Z$}.
\footnotetext{The word ``and'' does not establish the order between intervening and conditioning. As explained in \Cref{sec:cond-caus-entr}, we assume that interventions take precedence. See also the discussion in \Cref{sec:condIc-interpretation}.}
Accordingly, its definition will be the gap between the conditional entropy of $Y$ conditioned on $Z$ and the conditional causal entropy of $Y$ conditioned on $Z$ given an intervention protocol $X'$ for $X$.
\begin{definition}[Conditional Causal Information
  Gain] %
  \label[definition]{def:cond-info-gain}
  Let $X$, $Y$ and $Z$ be endogenous variables of an SCM $\mathcal{C}$.
  The \emph{conditional causal information gain} $I_{c}(Y \mid Z, \dop(X\sim X'))$ of $Y$ conditioned on $Z$ given the intervention protocol $X'$ for $X$ is defined as follows:
  \begin{equation}
    \label{eq:cond-info-gain}
    I_{c}(Y \mid Z, \dop(X\sim X')) \ccolon H(Y \mid Z) - H_{c}(Y \mid Z, \dop(X\sim X')).
  \end{equation}
\end{definition}

We can now leverage the chain rule for the causal entropy in
\Cref{prop:chain-rule-hc} to obtain a chain rule for the causal information gain.
\begin{proposition}[Chain Rule for Causal Information Gain]
  \label[proposition]{prop:chain-rule-Ic}
  Let $\mathbf{Y}$ be a random vector of length $n_{Y}$ and $\mathbf{Y}_{<i} = Y_1, \ldots, Y_{i-1}$.
  Then:
  \begin{equation}
    \label{eq:chain-rule-Ic}
    I_{c}(\mathbf{Y} \mid \dop(X \sim X)) = \sum_{i=1}^{n_{Y}} I_{c}(Y_{i} \mid \mathbf{Y}_{<i}, \dop(X \sim X)).
  \end{equation}
\end{proposition}

\subsection{Choices and Interpretations}
\label{sec:condIc-interpretation}
In the causal inference literature, if both the $\dop$ operator and a random variable appear after the conditioning bar, it is to be understood that the intervention precedes conditioning \citep{peters2017elements}.
We chose to respect this convention, so that here too interventions precede conditioning.
Indeed, the second term in \Cref{eq:cond-info-gain} relies on conditional causal entropy, which itself respects this convention.
Another choice was made by using $H(Y\mid Z)$ as the first term in \Cref{eq:cond-info-gain}.
Notice that this is the average entropy of $Y$ due to conditioning on $Z$, \emph{with respect to the pre-intervention joint distribution} $p_{Y, Z}$.
We will now look at the interpretation of conditional causal information gain and subsequenty introduce and interpret the quantity resulting from making a different choice for this first term, originating another reasonable causal generalization of conditional mutual information.

\Cref{def:cond-info-gain} tells us that $I_{c}(Y \mid Z, \dop(X \sim X'))$ is the information that is gained about $Y$ if one intervenes on $X$ before observing $Z$ as opposed to only observing $Z$.
In other words, it measures how much intervening on $X$ improves the information that is gained about $Y$ by observing $Z$.
\begin{example}
  \label{ex:contrast-agent}
  Radiologists often use substances called contrast agents before performing MRI (magnetic resonance imaging) scans to enhance image quality.
  Consult the causal graph in \Cref{fig:causal-graphs-examples}{(a)}.
  We want to assess the impact of using a contrast agent ($X=1$) on the information that can be extracted about the disease $Y$ from the MRI image $Z$.
  This is represented by the conditional causal information gain $I_{c}(Y \mid Z, \dop(X \sim X')) = H(Y \mid Z) - H(Y \mid Z, \dop(X = 1))$, where $X'$ was chosen to have a point mass at $1$.
  This precisely measures the information gained about disease $Y$ by employing the contrast agent before viewing the image $Z$, as opposed to solely observing $Z$ without the use of the contrast agent.

  In a case where $X$ has no total causal effect on $Z$ such as the one depicted in \Cref{fig:causal-graphs-examples}(b), $I_{c}(Y \mid Z, \dop(X \sim X'))$ can also be interpreted as the average information that is gained by intervening in strata of the population with the same value of $Z$.
  This comes about because in such situations the order between conditioning and intervening is irrelevant.
  In the case depicted this would be the information that is gained about the probability of patients having a stroke given that we intervene on their blood pressure, averaged over the age groups $Z=z$.
\end{example}
\begin{figure}[h]
  \centering
  \begin{subfigure}
    \centering
    \begin{tikzpicture}[mynode/.style={circle,draw=black,fill=white,inner sep=0pt,minimum size=0.8cm}, scale=0.8]
        \node[mynode] (x) at (0,0) {$X$};
          \node[yshift=+6.1pt] at (x.north) {\scriptsize Contrast Agent};
        \node[mynode, label={\scriptsize Disease}] (y) at (2,0) { $Y$};
        \node[mynode] (z) at (1,-1.75) {$Z$};
          \node[yshift=-5pt] at (z.south) {\scriptsize Image};
        \path [draw,->] (x) edge[-latex] (z);
        \path [draw,->] (y) edge[-latex] (z);
    \end{tikzpicture}
  \end{subfigure}
  \begin{subfigure}
    \centering
    \begin{tikzpicture}[mynode/.style={circle,draw=black,fill=white,inner sep=0pt,minimum size=0.8cm}, scale=0.8]
        \node[mynode, label={\scriptsize Blood Pressure}] (x) at (0,0) { $X$};
        \node[mynode, label={\scriptsize Stroke}] (y) at (2,0) {$Y$};
        \node[mynode] (z) at (1,-1.75) {$Z$};
          \node[yshift=-5pt] at (z.south) {\scriptsize Age};
        \path [draw,->] (z) edge[-latex] (x);
        \path [draw,->] (z) edge[-latex] (y);
        \path [draw,->] (x) edge[-latex] (y);
    \end{tikzpicture}
  \end{subfigure}
  \caption[Caption for cond Ic examples]{Causal graphs illustrating the interpretations of $I_{c}$ and $\mathrm{MI}_{c}$.}
  \label{fig:causal-graphs-examples}
\end{figure}
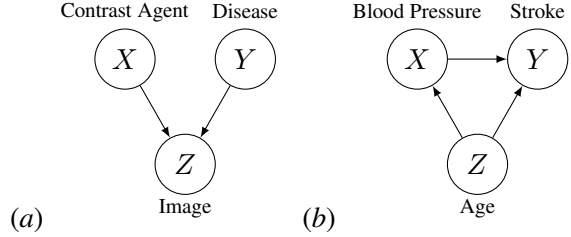

We will now see that replacing $H(Y \mid Z)$ with the causal entropy $H_{c}(Y \mid \dop(X \sim X'))$ in the first term of \Cref{eq:cond-info-gain} also results in a sensible quantity.
\begin{definition}
  \label[definition]{def:post-int-mutual-info}
  Let $X$, $Y$ and $Z$ be endogenous variables of an SCM $\mathcal{C}$.
  The \emph{post-intervention mutual information} $\mathrm{MI}_{c}(Y \mid Z, \dop(X \sim X'))$ between $Y$ and $Z$ given the intervention protocol $X'$ for $X$ is defined as follows:
  \begin{equation}
    \label{eq:post-int-mutual-info}
    \mathrm{MI}_{c}(Y \mid Z, \dop(X \sim X')) \ccolon H_{c}(Y \mid \dop(X\sim X')) - H_{c}(Y \mid Z, \dop(X \sim X')).
  \end{equation}
\end{definition}
This quantity measures the average information that is gained about $Y$ due to observing $Z$, given that $X$ was intervened on following the protocol $X'$.
There should therefore be a close connection between $\mathrm{MI}_{c}(Y \mid Z, \dop(X \sim X'))$ and the mutual information between $Y$ and $Z$ after one performs an atomic intervention on $X$\footnotemark.
Indeed, the post-intervention mutual intervention turns out to be the average mutual information between $Y$ and $Z$ after an atomic intervention on $X$:
\footnotetext{This is the reason for the name of this quantity.}
\begin{proposition}
  \label[proposition]{prop:mic-is-average-mi}
  Let $X'$ be an intervention protocol for $X$.
  Then $\mathrm{MI}_{c}(Y \mid Z, \dop(X \sim X')) = \E_{x\sim p_{X'}}\left[ I(Y ; Z \mid \dop(X = x)) \right]$, where the notation $I(Y ; Z \mid \dop(X=x))$ indicates that the mutual information is computed with respect to the joint post-atomic-intervention distribution.
\end{proposition}
\begin{example}
  We now revisit the scenario outlined in \Cref{ex:contrast-agent} and juxtapose the meanings of conditional causal information gain and post-intervention mutual information within this context.
  Recall that the intervention protocol had been chosen to be $X' \sim \delta(1)$.
  While $I_{c}(Y \mid Z, \dop(X \sim X'))$ quantifies the impact that the contrast agent $X$ has on the information that is gained about $Y$ by observing $Z$, $\mathrm{MI}_{c}(Y \mid Z, \dop(X \sim X')) =  I(Y ; Z \mid \dop(X = 1))$ is the information that is gained about the disease $Y$ by observing the image $Z$, given that the contrast agent has been injected.

  For the case of \Cref{fig:causal-graphs-examples}(b), $\mathrm{MI}_{c}(Y \mid Z, \dop(X \sim X'))$ is simply the average information that is gained about the probability of having a stroke by knowing the age, given that one intervened on blood pressure following a chosen intervention protocol.
\end{example}

\section{Related Work}
\label{sec:related-work}

We build upon the work developed in \citet{simoes2023causal}, which was itself inspired in work from the philophy of science literature \citep{griffiths2015specificity}.
\Citet{simoes2023causal} define causal counterparts of entropy and mutual information, referred to as causal entropy and causal information gain.
These metrics are designed to evaluate the extent to which a feature has control over a chosen outcome variable.
They do this by capturing changes in the entropy of a variable resulting from intervening on other variables.
The authors also study the relationship between these quantities and the existence of total causal effect.

Another causal generalization of mutual information had been proposed before  \citep{ay2008information}.
This quantity, termed ``information flow'', is both conceptually and numerically distinct from causal information gain:
Information flow was introduced both as a  measure of ``causal strength'' and of ``causal independence'', similarly to how standard mutual information is a measure of statistical independence.
This is accomplished by starting from the definition of mutual information as a KL divergence and proceeding to ``make it causal'' by replacing conditioning with interventions.
In contrast, \citet{simoes2023causal} treat entropy as the main quantity of interest.
They start from the definition of mutual entropy as the change in entropy due to conditioning, and define causal entropy as the change in entropy due to intervening.
This then results in a quantity that is appropriate for evaluating the control that a variable has over another, where control is taken to be how much one can reduce the uncertainty about the second by intervening on the first.
As confirmation that these two causal generalizations of mutual information are indeed distinct, one can simply notice that the former can be written as a KL divergence \citep{ay2008information}, while the latter cannot (\Cref{cor:Ic-can-be-negative}).
It should be noted  that there are metrics other than the ACE and the information flow which purport to measure strength.
See \citet{janzing2013quantifying} for a compilation of such metrics.
From those, the only one conceptually close to ours is information flow, given that it also serves as a causal version of mutual information.

\section{Discussion and Conclusion}
\label{sec:disc-concl}
The motivation for extending traditional entropy and mutual information to interventional settings stems from the desire to develop algorithms that utilize information theoretical quantities in the presence of non-causal statistical dependencies (\emph{e.g.} due to unobserved confounding).
Extending these quantities to handle interventions allows them to capture the effects of manipulating one variable on another.
The causal entropy, conditional causal entropy, causal information gain, and conditional causal information gain, together with their basic properties proved herein provide the foundation for developing new algorithms in areas where one has or can obtain knowledge of the causal relationships involved.

Establishing practical methods for computing these quantities remains a topic for future investigation.
This includes designing appropriate estimators, evaluating their performance characteristics (such as consistency, bias and rate of convergence), and determining whether these estimators are identifiable based on the available information about the underlying causal structure.
Furthermore, one can leverage the foundational results established in this paper (such as the chain rules) to aid the computation of causal control between variables in a complex structural causal model.
More generally, the properties presented herein can aid in examining how intervening on a variable within a structural causal model impacts the uncertainty associated with other variables.
As alluded to in the Introduction, potential applications of these concepts include guiding action selection in reinforcement learning, devising causal adaptations of entropy-based decision tree algorithms, and developing causal versions of representation learning algorithms which rely on information theoretical quantities such as the mutual information between the representation and a target variable.
On the theoretical front, one may try to generalize other important results from information theory to this causal information theoretical framework.
Additionally, other causal generalizations of conditional causal entropy and information gain can be studied, in particular those for which the assumption that interventions are performed before conditioning is dropped.
Lastly, the precise connection between causal information gain and putative measures of causal strength, such as information flow, could be studied in detail.

\section*{Acknowledgments}
This publication is part of the CAUSES project (KIVI.2019.004) of the research programme Responsible Use of Artificial Intelligence which is financed by the Dutch Research Council (NWO) and ProRail.

\bibliography{clear24library.bib}

\appendix

\section{Structural Causal Models}
\label{sec:struct-caus-models}

One can model the causal structure of a system by means of a ``structural causal model'', which can be seen as a Bayesian network \citep{koller2009probabilistic} whose graph $G$ has a causal interpretation and each conditional probability distribution (CPD) $P(X_{i} \mid \PA_{X_{i}})$ of the Bayesian network stems from a deterministic function $f_{X_{i}}$ (called ``structural assignment'') of the parents of $X_{i}$.
In this context, it is common to separate the parent-less random variables (which are called ``exogenous'' or ``noise'' variables) from the rest (called ``endogenous'' variables).
Only the endogenous variables are represented in the structural causal model graph.
As is commonly done \citep{peters2017elements}, we assume that the noise variables are jointly independent and that exactly one noise variable $N_{X_{i}}$ appears as an argument in the structural assignment $f_{X_{i}}$ of $X_{i}$.
In full rigor\footnotemark \citep{peters2017elements}:

\footnotetext{We slightly rephrase the definition provided in \citet{peters2017elements} for our purposes. \label{fn:def}}

\begin{definition}[Structural Causal Model]
  \label[definition]{def:scm}
  Let $X$ be a random variable with range $R_{X}$ and $\mathbf{W}$ a random vector with range $R_{\mathbf{W}}$.
  A \emph{structural assignment for $X$ from $\mathbf{W}$} is a function $f_{X}\colon R_{\mathbf{W}} \to R_{X}$.
  A \emph{structural causal model (SCM)} $\mathcal{C} = (\mathbf{X}, \mathbf{N}, S, p_{\mathbf{N}})$ consists of:
  \begin{enumerate}
    \item A random vector $\mathbf{X} = (X_{1}, \ldots, X_{n})$ whose variables we call \emph{endogenous}.
    \item A random vector $\mathbf{N} = (N_{X_{1}}, \ldots, N_{X_{n}})$ whose variables we call \emph{exogenous} or \emph{noise}.
    \item A set $S$ of $n$ structural assignments $f_{X_{i}}$ for $X_{i}$ from ($\PA_{X_{i}}, N_{X_{i}}$), where $\PA_{X_{i}} \subseteq \mathbf{X}$ are called \emph{parents} of $X_{i}$.
      The \emph{causal graph} $G^{\mathcal{C}}\vcentcolon=(\mathbf{X}, E)$ of $\mathcal{C}$ has as its edge set $E = \{(P, X_{i}) : X_{i} \in \mathbf{X},\  P\in \PA_{X_{i}}\}$.
      The $\PA_{X_{i}}$ must be such that the $G^{\mathcal{C}}$ is a directed acyclic graph (DAG).
    \item A jointly independent probability distribution $p_{\mathbf{N}}$ over the noise variables. We call it simply the \emph{noise distribution}.
  \end{enumerate}
\end{definition}

  We denote by $\cC(\mathbf{X})$ the set of SCMs with vector of endogenous variables $\mathbf{X}$.
  Furthermore, we write $X \vcentcolon= f_{X}(X, N_{X})$ to mean that $f_{X}(X, N_{X})$ is a structural assignment for $X$.

  Notice that for a given SCM the noise variables have a known distribution $p_{\mathbf{N}}$ and the endogenous variables can be written as functions of the noise variables.
  Therefore the distributions of the endogenous variables are themselves determined if one fixes the SCM.
  This brings us to the notion of the entailed distribution\footref{fn:def} \citep{peters2017elements}:

\begin{definition}[Entailed distribution]
  Let $\mathcal{C} = (\mathbf{X}, \mathbf{N}, S, p_{\mathbf{N}})$ be an SCM. Its \emph{entailed distribution} $p^{\mathcal{C}}_{\mathbf{X}}$  is the unique joint distribution over $\mathbf{X}$ such that $\forall X_{i} \in \mathbf{X},\ X_{i} = f_{X_{i}}(\PA_{X_{i}}, N_{X_{i}})$.
  It is often simply denoted by $p^{\cC}$.
  Let $\mathbf{x}_{-i}\vcentcolon= (x_{1}, \ldots, x_{i-1}, x_{i+1}, \ldots, x_{n})$.
  For a given $X_{i} \in \mathbf{X}$, the marginalized distribution $p^{\cC}_{X_{i}}$ given by $p^{\cC}_{X_{i}}(x_{i}) = \sum_{\mathbf{x}_{-i}} p^{\cC}_{\mathbf{X}}(\mathbf{x})$ is also referred to as \emph{entailed distribution (of $X_{i}$)}.
\end{definition}

Having an SCM in hand allows us to model interventions on the system. The idea
is that an SCM represents how the values of the random variables are generated,
and by intervening on a variable system we are effectively changing its
generating process -- read: its structural assignment. Thus intervening on a
variable can be modeled by modifying the structural assignment of said variable,
resulting in a new SCM differing from the original only in the structural
assignment of the intervened variable, and possibly introducing a new noise
variable for it, in place of the old one. Naturally, the new SCM will have an
entailed distribution which is in general different from the distribution
entailed by the original SCM.

\begin{definition}[General intervention]
  Let $\cC = (\mathbf{X}, \mathbf{N}, S, p_{\mathbf{N}})$ be an SCM,
  $X_{i} \in \mathbf{X}$, $\widetilde{\PA}_{X_{i}} \subseteq \mathbf{X}$,
  $\tilde{N}_{i}$ be a random variable and $\tilde{f}_{X_{i}}$ be a structural
  assignment for $X_{i}$ from $\widetilde{\PA}_{X_{i}}, \tilde{N}_{i}$.

  The \emph{intervention
    $\dop(X_{i} = \tilde{f}_{X_{i}}(\widetilde{\PA}_{X_{i}}, \tilde{N}_{i}))$}
  is the function $\cC(\mathbf{X}) \to \cC(\mathbf{X})$ given by
  $\cC \mapsto \cC^{\dop(X_{i} = \tilde{f}_{X_{i}}(\widetilde{\PA}_{X_{i}}, \tilde{N}_{i}))}$,
  where
  $\cC^{\dop(X_{i} = \tilde{f}_{X_{i}}(\widetilde{\PA}_{X_{i}}, \tilde{N}_{i}))}$
  is the ordered pair
  $(\mathbf{X}, \tilde{\mathbf{N}}, \tilde{S}, p_{\tilde{N}})$, with
  $\tilde{\mathbf{N}} = (\mathbf{N}\setminus \{N_{i}\}) \cup \{\tilde{N}_{i}\}$
  and $\tilde{S} = (S \setminus \{f_{X_{i}}\}) \cup \{\tilde{f}_{X_{i}}\}$. We
  call it the \emph{post-intervention SCM (w.r.t. the intervention
    $\dop(X_{i} = \tilde{f}_{X_{i}}(\widetilde{\PA}_{X_{i}}, \tilde{N}_{i}))$)}.
  It is also denoted
  $\tilde{\cC}:=\cC^{\dop(X_{i} = \tilde{f}_{X_{i}}(\widetilde{\PA}_{X_{i}}, \tilde{N}_{i}))}$.

  Note that in order for $\tilde{\cC}$ to be an SCM, $\widetilde{\PA}_{X_{i}}$
  must be such that the causal graph $G^{\tilde{\cC}}$ is a DAG.

  We then say that the variable $X_{i}$ was \emph{intervened on}.

  The distribution
  $p^{\dop(X_{i} = \tilde{f}_{X_{i}}(\widetilde{\PA}_{X_{i}}, \tilde{N}_{i}))} := p^{\tilde{\cC}}$
  entailed by $\tilde{\cC}$ is called the \emph{post-intervention distribution
    (w.r.t. the intervention
    $\dop(X_{i} = \tilde{f}_{X_{i}}(\widetilde{\PA}_{X_{i}}, \tilde{N}_{i}))$ on
    $\cC$)}.
\end{definition}

The most common type of interventions are the so-called ``atomic
interventions'', where one sets a variable to a chosen value, effectively
replacing the distribution of the intervened variable with a point mass
distribution. In particular, this means that the intervened variable has no
parents after the intervention.
\begin{definition}[Atomic intervention]
  Let $\cC = (\mathbf{X}, \mathbf{N}, S, p_{\mathbf{N}})$ be an SCM and
  $X_{i} \in \mathbf{X}$. An \emph{atomic intervention on $X_{i}$} is an
  intervention of the type $\dop(X_{i} = \tilde{N}_{i})$, where $\tilde{N}_{i}$
  is a random variable with range $R_{X_{i}}$ and
  $p_{\tilde{N_{i}}}(x_{i}) = \delta_{x, x_{i}}$ for some $x\in R_{X_{i}}$. Such
  an intervention is usually denoted simply by $\dop(X_{i} = x)$.
\end{definition}

Another special type of intervention is the ``stochastic intervention''
\citep{korb2004varieties}, where again the intervened variable has no parents,
but its distribution can be any distribution. Thus, an atomic intervention is a
particular type of stochastic intervention.
\begin{definition}[Stochastic intervention]
  Let $\cC = (\mathbf{X}, \mathbf{N}, S, p_{\mathbf{N}})$ be an SCM and
  $X_{i} \in \mathbf{X}$. A \emph{stochastic intervention on $X_{i}$} is an
  intervention of the type $\dop(X_{i} = \tilde{N}_{i})$, where $\tilde{N}_{i}$
  is a random variable with range $R_{X_{i}}$ and $p_{\tilde{N_{i}}}$ can be any
  probability distribution.
\end{definition}

We can also define what we mean by ``$X$ having a total causal effect on $Y$''.
Following \citet{peters2017elements,pearl2009causality}, there is such a total
causal effect if there is an atomic intervention on $X$ which modifies the
initial distribution of $Y$.

\begin{definition}[Total Causal Effect]
  Let $X$, $Y$ be random variables of an SCM $\mathcal{C}$. $X$ has a
  \emph{total causal effect on} $Y$ if there is $x\in R_{X}$ such that
  $p^{\dop(X=x)}_{Y} \ne p_{Y}$. We then write $X \veryshortrightarrow Y$.
\end{definition}

\section{Relating Stochastic and Atomic Post-intervention Distributions}
We will here prove a lemma that is used to prove many of the results in this
paper.%
\begin{lemma}[Atomic Intervention Equals Conditioning After Stochastic Intervention]
  \label[lemma]{lemma:atomic-vs-stochastic}
  Let $X$ be an endogenous variable of an SCM $\mathcal{C}$ and $\mathbf{Y}$ be a vector of endogenous variables of $\mathcal{C}$ distinct from $X$.
  Furthermore, let $X'$ be an intervention protocol for $X$ and $x$ be an $X$-value in the support of $p_{X'}$.
  The post-intervention distribution of $\mathbf{Y}$ resulting from an atomic intervention
  $\dop(X=x)$ equals the conditional post-intervention distribution of $\mathbf{Y}$
  resulting from the stochastic intervention $\dop(X=X')$ and conditioning on
  $X=x$. That is:
  \begin{equation}
    \label{eq:atomic-vs-stochastic}
    p_{\mathbf{Y}}^{\dop(X=x)} = p_{\mathbf{Y}\mid X=x}^{\dop(X=X')}.
  \end{equation}
\end{lemma}
\begin{proof}
  This proof will be easier if we introduce operators corresponding to
  intervening and conditioning. That will allow us to not overburden our
  notation with subscripts and superscripts.
  Denote by $\mathbf{X}$ the set of endogenous variables of $\cC$, and by
  $\mathcal{M}_{\mathcal{P}(\mathbf{X})}$ the set of all probability mass
  functions for any of the variable subsets in the powerset
  $\mathcal{P}(\mathbf{X})$ of $\mathbf{X}$.
  Let $x$ be an $X$-value. Define operators
  $\mathrm{Do}[X=x]\colon \mathcal{M}_{\mathcal{P}(\mathbf{X})} \to \mathcal{M}_{\mathcal{P}(\mathbf{X})}$
  and
  $\mathrm{Cond}[X=x]\colon \mathcal{M}_{\mathcal{P}(\mathbf{X})} \to \mathcal{M}_{\mathcal{P}(\mathbf{X})}$
  such that
  $\forall \mathbf{Y} \subseteq \mathbf{X}\backslash X,\ \mathrm{Do}[X=x](p_{\mathbf{Y}}) = p^{\dop(X=x)}_{\mathbf{Y}}$ and
  $\forall \mathbf{Y} \subseteq \mathbf{X}\backslash X,\ \mathrm{Cond}[X=x](p_{\mathbf{Y}}) = p_{\mathbf{Y} \mid X=x}$.
  Let $\mathbf{Y} \subseteq \mathbf{X}\backslash X$.
  Then what we want to prove can be written
  \begin{equation}
    \label{eq:lemma_equivalent}
    \mathrm{Cond}[X=x] \left(\mathrm{Do}[X=X'] \left(p_{\mathbf{Y}} \right) \right)
    = \mathrm{Do}[X=x]  \left(p_{\mathbf{Y}} \right).
  \end{equation}
  Denote by $f_{X}$ the structural assignment for $X$ in $\mathcal{C}$. The only
  difference between $\mathcal{C}^{\dop(X=X')}$ and $\mathcal{C}$ is that
  $f_{X}(\PA_{X}, N_{X})$ is replaced by another structural assignment
  $\tilde{f}_{X}(\tilde{N}_{X} = X') = X'$ and there is a new variable $X'$ in
  $\mathcal{C}^{\dop(X=X')}$ which did not figure in $\mathcal{C}$. If one
  proceeds by performing the intervention $\dop(X=x)$, one obtains the SCM
  $(\mathcal{C}^{\dop(X=X')})^{\dop(X=x)}$ which differs from
  $\mathcal{C}^{\dop(X=X')}$ only in that $\tilde{f}_{X}(X')$ is replaced by
  $\tilde{\tilde{f}}_{X}(\tilde{\tilde{N}}_{X} \sim \delta_{x}) = \tilde{\tilde{N}}_{X}$.
  Notice that the marginal distribution
  $p_{\mathbf{Y}}^{(\mathcal{C}^{\dop(X=X')})^{\dop(X=x)}}$ entailed by this SCM is
  precisely the result of
  $\mathrm{Do}[X=x]\left(\mathrm{Do}[X=X'](p_{\mathbf{Y}})\right)$. Furthermore, the SCM
  entailing the RHS of \eqref{eq:lemma_equivalent} differs from
  $(\mathcal{C}^{\dop(X=X')})^{\dop(X=x)}$ only in that the latter contains a
  childless exogenous variable $X'$. Therefore they entail the same marginal
  distribution of $\mathbf{Y}$. Hence:
  $$\mathrm{Do}[X=x]\left(\mathrm{Do}[X=X'](p_{\mathbf{Y}})\right) = \mathrm{Do}[X=x]  \left(p_{\mathbf{Y}} \right).$$
  On the other hand, starting again with $\mathcal{C}$ and setting $X=X'$ sets
  $\PA_{X} = \emptyset$, which means that intervening on $X$ after that will be
  equivalent to conditioning:
  $$\mathrm{Do}[X=x] \left(\mathrm{Do}[X=X'] (p_{\mathbf{Y}}) \right)
  = \mathrm{Cond}[X=x] \left(\mathrm{Do}[X=X'] (p_{\mathbf{Y}}) \right).$$
  Hence \eqref{eq:lemma_equivalent} holds, which proves the lemma.
\end{proof}

\section{Proofs for Results on Causal Entropy}

\subsection*{Computations for \Cref{ex:h_differs_from_hc}}
\label{comp:h-from-table}

\begin{equation}
  \begin{cases}
    p^{\dop(X=0)}_{Y}(0) = p_{N_{Y}}(0) = \nicefrac{1}{2} \\
    p^{\dop(X=0)}_{Y}(1) = p_{N_{Y}}(1) = \nicefrac{1}{2} \\
    p^{\dop(X=0)}_{Y}(2) = 0
  \end{cases}
\end{equation}

\begin{equation}
  \begin{cases}
    p^{\dop(X=1)}_{Y}(0) = 0 \\
    p^{\dop(X=1)}_{Y}(1) = p_{N_{Y}}(0) = \nicefrac{1}{2} \\
    p^{\dop(X=1)}_{Y}(2) = p_{N_{Y}}(1) = \nicefrac{1}{2}
  \end{cases}
\end{equation}

\begin{equation}
  \begin{cases}
    p^{\dop(X=X')}_{Y}(0) = p_{X'}(0)p_{N_{Y}}(0) = \frac{2}{3} \times \frac{1}{2} = \nicefrac{1}{3} \\
    p^{\dop(X=X')}_{Y}(1) = p_{X'}(0) p_{N_{Y}}(1) + p_{X'}(1) p_{N_{Y}}(0) = \nicefrac{1}{2} \\
    p^{\dop(X=X')}_{Y}(2) = p_{X'}(1)p_{N_{Y}}(1) = \frac{1}{3} \times \frac{1}{2} = \nicefrac{1}{6}
  \end{cases}
\end{equation}

\begin{align}
  \begin{split}
    H_{c}(Y\mid \dop(X\sim X')) =& - p_{X'}(0) \bigg(p_{Y}^{\dop(X=0)}(0) \log p_{Y}^{\dop(X=0)}(0) \\
                                 &\hspace{1.8cm} + p_{Y}^{\dop(X=0)}(1) \log p_{Y}^{\dop(X=0)}(1) \\
                                 &\hspace{1.8cm} + p_{Y}^{\dop(X=0)}(2) \log p_{Y}^{\dop(X=0)}(2)\bigg) \\
                                 &- p_{X'}(1) \bigg(p_{Y}^{\dop(X=1)}(0) \log p_{Y}^{\dop(X=1)}(0) \\
                                 &\hspace{1.8cm} + p_{Y}^{\dop(X=1)}(1) \log p_{Y}^{\dop(X=1)}(1)  \\
                                 &\hspace{1.8cm}+p_{Y}^{\dop(X=1)}(2) \log p_{Y}^{\dop(X=1)}(2)\bigg) \\
    =& - 2/3 \times (-1/2 - 1/2) - 1/3 \times  (0 - 1/2 - 1/2) = 1 (\mathrm{bit})
  \end{split}
\end{align}

\begin{align}
  \begin{split}
    H(Y\mid \dop(X=X')) =& p^{\dop{(X=X')}}_{Y}(0) \log(\nicefrac{1}{p_{Y}^{\dop(X=X')}(0)})\\
                         &\hspace{0.2cm} + p^{\dop{(X=X')}}_{Y}(1) \log(\nicefrac{1}{p_{Y}^{\dop(X=X')}(1)}) \\
                         &\hspace{0.2cm} + p^{\dop{(X=X')}}_{Y}(2) \log(\nicefrac{1}{p_{Y}^{\dop(X=X')}(2)}) \\
    =&\frac{1}{3}\log 3 + \frac{1}{2} + \frac{1}{6}\log 6.
  \end{split}
\end{align}

\begin{proof}[Proof of \Cref{prop:cH-single-E}]
  \label{proof:cH-single-E}
  \begin{align}
    H_{c}(Y\mid \dop(X\sim X')) & \equiv - \E_{x'\sim p_{X'}}\left[ \E_{y\sim p_{Y}^{\dop(X=x')}}\left[ \log p_{Y}^{\dop(X=x')}(y) \right]  \right]  \\
                                & \equiv - \sum_{x',y} p_{X'}(x') p_{Y}^{\dop(X=x')}(y)  \log p_{Y}^{\dop(X=x')}(y) \\
                                & = - \sum_{x',y} p^{\dop(X=X')}_{X}(x') p_{Y\mid X=x'}^{\dop(X=X')}(y)  \log p_{Y}^{\dop(X=x')}(y)
                                  \label{eq:hc_critical_step} \\
                                & = - \sum_{x,y} p_{X,Y}^{\dop(X=X')}(x,y)  \log p_{Y}^{\dop(X=x')}(y) \\
                                & = - \E_{x,y \sim p_{X,Y}^{\dop(X=X')}} \left[ \log p(y \mid \dop(X=x)) \right]
                                  \label{eq:hc_final_expression}
  \end{align}
  where to get \Cref{eq:hc_critical_step} we used that
  $p_{X'}(x) =p ^{\dop(X=X')}_{X}(x)$ and \Cref{lemma:atomic-vs-stochastic}.
\end{proof}

\begin{proof}[Proof of \Cref{prop:Hc-as-E-cov-eff}]
  \label{proof:Hc-as-E-cov-eff}
  This comes directly from the definition of $H_{c}$ and
  \Cref{lemma:atomic-vs-stochastic}:
  \begin{align}
    \begin{split}
      H_{c}(Y\mid \dop(X &\sim  X')) = \E_{x\sim p_{X'}} \left[ H_{Y \sim p^{\dop(X=x)}_{Y}}(Y) \right] \\
                        &= \E_{x\sim p_{X'}} \left[ H_{Y\sim p_{Y\mid X=x}^{\dop(X=X')}}(Y) \right].
    \end{split}
  \end{align}
\end{proof}

\begin{proof}[Proof of \Cref{prop:hc-nonnegative}] %
  The result follows directly from the definition of causal entropy. We can
  write:
  \begin{equation}
    H_{c}(Y\mid \dop(X \sim X'))
    = \E_{x'\sim p_{X'}} \left[ H_{Y \sim p^{\dop(X=x')}_{Y}}(Y) \right].
  \end{equation}
  Since entropies are non-negative quantities, it follows that the causal
  entropy, being an average of entropies, is also non-negative.
\end{proof}

\begin{proof}[Proof of
  \Cref{prop:Hc-lessthan-H}] %
  \label{proof:Hc-lessthan-H}
  It suffices to provide an SCM $\mathcal{C}$ and intervention protocol $X'$
  such that the causal entropy of $Y$ given the intervention protocol $X'$ for
  $X$ is greater than the initial entropy of $Y$.
  Consider an SCM $\cC$ with exactly two binary endogenous variables $X, Y$, causal graph $X \rightarrow Y$ and structural assignments given by:
  \begin{equation}
    \begin{cases}
      X = N_{X} \\
      Y =
      \begin{cases}
        y_{0}, X = x_{0} \\
        N_{Y}, X = x_{1}
      \end{cases}
    \end{cases}.
  \end{equation}
  Furthermore, assume that $N_{Y} \sim \mathrm{Bern}(\nicefrac{1}{2})$, $N_{X} \sim \mathrm{Bern}(\nicefrac{1}{10})$ and $X' \sim \mathrm{Bern{(\nicefrac{1}{2})}}$.
  Notice that this implies:
  \begin{equation}
    \begin{cases}
      p_{Y}^{\dop(X=x_{0})} = p_{Y \mid X=x_{0}} = \delta(x_{0}) \\
      p_{Y}^{\dop(X=x_{1})} = p_{Y \mid X=x_{1}} = \mathrm{Bern}(\nicefrac{1}{2}) \\
    \end{cases}.
  \end{equation}
  Then:
  \begin{align}
    \begin{split}
      H_{c}(Y \mid \dop(X \sim X')) &= \E_{x \sim p_{X'}} \left[ H_{Y \sim p^{\dop(X=x)}_{Y}}(Y)\right] \\
                                    &= \E_{x \sim p_{X'}} \left[ H_{Y \sim p_{Y \mid X=x}}(Y)\right] \\
                                    &= p_{X'}(x_{0}) \underbrace{H_{Y \sim p_{Y \mid X = x_{0}}}(Y)}_{0} + p_{X'}(x_{1}) \underbrace{H_{Y \sim p_{Y \mid X = x_{1}}}(Y)}_{1} \\
                                    &= p_{X'}(x_{1}) = \nicefrac{1}{2}
    \end{split}
  \end{align}
  where in the second equality we used that the causal effect from $X$ to $Y$ is
  not confounded.
  \begin{align}
    \begin{split}
      H(Y) &= -\sum_{y} p_{Y}(y) \log\left(p_{Y}(y) \right) \\
          &= -\sum_{x, y} p_{Y \mid X=x}(y) p_{X}(x) \log \left(\sum_{\dot{x}} p_{Y \mid X = \dot{x}}(y) p_{X}(\dot{x})\right) \\
          &= -\sum_{x} p_{X}(x) \sum_{y} p_{Y \mid X=x}(y) \log \Big( p(y \mid x_{0}) p_{X}(x_{0}) + p(y \mid x_{1}) p_{X}(x_{1}) \Big) \\
          &= -p_{X}(x_{0}) \log \Big( p_{X}(x_{0}) + \frac{1}{2} p_{X}(x_{1}) \Big) \\
           &\qquad - \frac{1}{2} p_{X}(x_{1}) \Big[\log \Big( p_{X}(x_{0}) + \frac{1}{2} p_{X}(x_{1}) \Big) + \log \Big(\frac{1}{2}p_{X}(x_{1})\Big)\Big] \approx 2.04.
    \end{split}
  \end{align}
  Hence $H_{c}(Y \mid X \sim X') > H(Y)$ for this SCM and intervention protocol.
\end{proof}

\begin{proof}[Proof of \Cref{cor:Hc-lessthan-stochH}]
  Conditional entropy cannot be greater than the entropy before conditioning.
  By \Cref{prop:Hc-as-E-cov-eff} causal entropy is a conditional entropy obtained from the post-stochastic intervention entropy by conditioning.
  The result follows.
\end{proof}

\begin{proof}[Proof of \Cref{prop:ind-bound}]
  \label{proof:ind-bound}
  We make use of \Cref{prop:Hc-as-E-cov-eff} and the independence bound on
  (standard) entropy:
  \begin{align}
    \begin{split}
      H_{c}(\mathbf{Y} \mid \dop(X \sim X')) &= H(\mathbf{Y} \mid X, \dop(X = X')) \\
                                                              &\le  \sum_{i=1}^{n_{Y}}H(Y_{i} \mid X, \dop(X = X')) \\
                                                              &=  \sum_{i=1}^{n_{Y}}H_{c}(Y_{i} \mid \dop(X \sim X')).
    \end{split}
  \end{align}
\end{proof}

\begin{proof}[Proof of \Cref{prop:cond-hc-as-cond-entr}]
  \label{proof:cond-hc-as-cond-entr}
  \begin{align}
    \begin{split}
      H_{c}(Y &\mid Z, \dop(X\sim X')) = \E_{x\sim p_{X'}} \left[ \E_{z\sim p^{\dop(X=x)}_{Z}}\left[H_{Y \sim p^{\dop(X=x)}_{Y \mid Z=z}}(Y)\right] \right] \\
              &= \sum_{x, z} p_{Z}^{\dop(X=x)}(z) p_{X'}(x) H_{Y \sim p^{\dop(X=x)}_{Y \mid Z=z}}(Y) \\
              &= \sum_{x, z} p_{Z \mid X=x}^{\dop(X=X')}(z) p_{X}^{\dop(X=X')}(x) H_{Y \sim p^{\dop(X=X')}_{Y \mid X=x, Z=z}}(Y) \\
              &= \sum_{x, z} p_{X,Z}^{\dop(X=X')}(x,z) H(Y \mid X=x, Z=z, \dop(X=X')) \\
              &= H(Y \mid X, Z, \dop(X=X'))
    \end{split}
  \end{align}
  where in the third step we used \Cref{lemma:atomic-vs-stochastic} twice, and in the
  last one we used the definition of conditional entropy.
\end{proof}

\begin{proof}[Proof of \Cref{prop:cond-reduces-hc}]
  \label{proof:cond-reduces-hc}
  We will use \Cref{prop:cond-hc-as-cond-entr} and \Cref{prop:Hc-as-E-cov-eff}
  together with the fact that conditional entropy can never be larger than the
  initial entropy \citep{thomas2006information} to prove the result:
  \begin{align}
    \begin{split}
      H_{c}(Y \mid Z, \dop(X \sim X')) &= H(Y \mid Z, X, \dop(X=X')) \\
                                      &\le H(Y \mid X, \dop(X=X')) \\
                                      &= H_{c}(Y \mid \dop(X \sim X')).
    \end{split}
  \end{align}
\end{proof}

\begin{proof}[Proof of \Cref{prop:hc-2-var-chain-rule}]
  \label{proof:hc-2-var-chain-rule}
  Due to \Cref{prop:Hc-as-E-cov-eff} and \Cref{prop:cond-hc-as-cond-entr} we can leverage the chain
  rule for entropy to obtain a chain rule for the causal entropy:
  \begin{align}
    \begin{split}
      H_{c}(Y, Z \mid \dop(X\sim X')) &= H(Y, Z \mid X, \dop(X=X')) \\
                                      &= H(Y \mid X, \dop(X=X')) + H(Z \mid Y, X, \dop(X=X'))\\
                                      &= H_{c}(Y \mid \dop(X\sim X')) + H_{c}(Z \mid Y, \dop(X \sim X')).
    \end{split}
  \end{align}
\end{proof}

\begin{proof}[Proof of
  \Cref{prop:chain-rule-hc}] %
  \label{proof:chain-rule-hc}
  \begin{align}
    \begin{split}
      H_{c}(\mathbf{Y} \mid \dop(X \sim X')) &= H(\mathbf{Y} \mid X, \dop(X = X')) \\
                                                              &= \sum_{i=1}^{n_{Y}} H(Y_{i} \mid \mathbf{Y}_{<i}, X, \dop(X = X')) \\
                                                              &= \sum_{i=1}^{n_{Y}} H_{c}(Y_{i} \mid \mathbf{Y}_{<i}, \dop(X \sim X')) \\
    \end{split}
  \end{align}
  where in the first step and third steps we used
  \Cref{prop:cond-hc-as-cond-entr},
  while the second equality follows from the chain rule for entropy.
\end{proof}

\section{Proofs for Results on Causal Information Gain}

\begin{proof}[Proof of
  \Cref{prop:chain-rule-Ic}] %
  \begin{align}
    \begin{split}
      I_{c}(\mathbf{Y} \mid \dop(X \sim X')) &= H(\mathbf{Y}) - H_{c}(\mathbf{Y} \mid \dop(X\sim X')) \\
                                                              &= \sum_{i=1}^{n_{Y}} \Big( H(Y_{i} \mid \mathbf{Y}_{<i}) -  H_{c}(Y_{i} \mid \mathbf{Y}_{<i}, \dop(X \sim X')) \Big) \\
                                                              &=  \sum_{i=1}^{n_{Y}} I_{c}(Y_{i} \mid \mathbf{Y}_{<i}, \dop(X \sim X'))
    \end{split}
  \end{align}
  where in the second step we used the chain rules for the entropy
  \citep{thomas2006information} and for the causal entropy
  (\Cref{prop:chain-rule-hc}).
\end{proof}

\begin{proof}[Proof of \Cref{prop:mic-is-average-mi}]
  \begin{align}
    \begin{split}
      \mathrm{MI}_{c}(Y \mid Z, \dop(X\sim X')) &= \E_{x\sim p_{X'}} \left[ H(Y\mid \dop(X = x)) \right] \\
                                               &\qquad - \E_{x\sim p_{X'}} \left[H(Y \mid Z, \dop(X = x)) \right] \\
        &=\E_{x\sim p_{X'}}\left[ H(Y\mid \dop(X = x)) - H(Y \mid Z, \dop(X = x)) \right] \\
        &= \E_{x\sim p_{X'}}\left[I(Y ; Z \mid \dop(X=x))\right]
    \end{split}
  \end{align}
\end{proof}

\end{document}